\documentclass[twoside]{article}

\usepackage[accepted]{aistats2019}
\usepackage[utf8]{inputenc} 
\usepackage[T1]{fontenc}    
\usepackage{hyperref}       
\usepackage{url}            
\usepackage{booktabs}       
\usepackage{amsfonts}       
\usepackage{nicefrac}       
\usepackage{microtype}      

\usepackage{times}
\usepackage{graphicx} 
\usepackage{caption}
\usepackage{subcaption}

\usepackage{algorithm}
\usepackage{algorithmic}
\usepackage{amsfonts}
\usepackage{amsthm}
\usepackage{amsmath}
\newcommand\floor[1]{\lfloor#1\rfloor}
\newcommand\ceil[1]{\lceil#1\rceil}

\usepackage[usenames,dvipsnames]{xcolor}
\usepackage[backgroundcolor=White,textwidth=0.7in]{todonotes}
\usepackage{setspace}

\newtheorem{theorem}{Theorem}
\newtheorem{corollary}{Corollary}
\newtheorem{lemma}{Lemma}
\newtheorem{remark}{Remark}
\newtheorem{assu}{Assumption}

\usepackage{hyperref}

\begin{document}

%

%

\twocolumn[

\aistatstitle{Nearly Optimal Adaptive Procedure with Change Detection for Piecewise-Stationary Bandit}

\aistatsauthor{ Yang Cao \And Zheng Wen \And Branislav Kveton \And Yao Xie }

\aistatsaddress{ Uber Technologies Inc. \And  Adobe Research \And Google Research \And Georgia Tech} 

%
]

\begin{abstract}
Multi-armed bandit (MAB) is a class of online learning problems where a learning agent aims to maximize its expected cumulative reward while repeatedly selecting to pull arms with unknown reward distributions. We consider a scenario where the reward distributions may change in a piecewise-stationary fashion at unknown time steps. We show that by incorporating a simple change-detection component with classic UCB algorithms to detect and adapt to changes, our so-called M-UCB algorithm can achieve nearly optimal regret bound on the order of $O(\sqrt{MKT\log T})$, where $T$ is the number of time steps, $K$ is the number of arms, and $M$ is the number of stationary segments. 
Comparison with the best available lower bound shows that our M-UCB is nearly optimal in $T$ up to a logarithmic factor. 
We also compare M-UCB with the state-of-the-art algorithms in numerical experiments using a public Yahoo! dataset to demonstrate its superior performance. 
\end{abstract}

\setstretch{1}


\section{Introduction}
\label{sec:introduction}


Multi-armed bandit (MAB) is a class of fundamental problems in online learning and sequential decision making, where at each step a learning agent adaptively selects to pull one arm of a $K$-arm bandit based on its past observations, and receives one reward accordingly. 
The learning agent's objective is to maximize its expected cumulative reward in the first $T$ time steps.
MAB has found an extensive  list of applications including communication systems \cite{thompson1933likelihood,alaya2008dynamic}, clinical trials \cite{vermorel2005multi, villar2015multi}, online recommendation systems \cite{li2011unbiased, bouneffouf2012contextual, kveton2014matroid}, and online advertisement campaign \cite{girgin2012managing, schwartz2017customer}. 

Most existing literature on MAB problems focuses on two types of models:
(i) the stochastic bandit model \cite{lai1985asymptotically,auer2002finite}, where each of the $K$ arms has a time-invariant reward distribution, and (ii) the adversarial bandit model \cite{littlestone1994weighted,auer2002nonstochastic}, where the reward distribution of each arm may change adversarially at all the time steps. However, in many real-world applications, neither of the above two models is realistic. Specifically, in such applications, the arms' reward distributions do vary with time, but much less frequently than what the adversarial bandit model assumes.
For instance, in recommender systems, each item is modeled as an arm and users' clicks are modeled as rewards. In practice, a user's click probability on an item is unlikely to be time-invariant, or change significantly at all the time steps.
Thus, for this case, it is too ideal to assume the stochastic bandit model and too conservative to assume the adversarial bandit model. 
Similar situations arise in dynamic pricing systems and investment options selection \cite{yu2009piecewise, cesa2006prediction}. Motivated by this, we examine a scenario that lies ``in between'' the above two standard models, namely, the piecewise-stationary bandit that we describe below. The piecewise-stationary reward functions can also be viewed as an approximation to the slowly time-varying reward functions. 

In this paper, we consider a class of non-stationary bandit problems, where the reward distribution of each arm is piecewise-constant and shifts at some unknown time steps called the \textit{change-points}. This setting has been considered in prior works \cite{hartland2007change, garivier2008upper, yu2009piecewise} as a more realistic scenario to model the users' preferences and in \cite{auer2002using} to model an adversarial setting. We propose a simple but efficient algorithm called Monitored-UCB (M-UCB) by incorporating a change-point detection component into a classic Upper Confidence Bound (UCB) algorithm. M-UCB monitors the estimated mean of the reward distribution for the currently selected arm; once a change is detected, M-UCB algorithm will reset and learn the new optimal arm. 

We show that, somewhat surprisingly, this simple M-UCB algorithm is nearly optimal for the considered scenario, in the sense that it achieves an
$O(\sqrt{MKT\log T})$ 
regret bound under mild technical assumptions (see Section~\ref{sec:theory}),
where $T$ is the number of time steps, $K$ is the number of arms, and $M$ is the number of stationary segments.
This regret bound matches the $\Omega(\sqrt{T})$ lower bound  proven in \cite{garivier2011upper} up to a logarithmic factor. In practice, M-UCB is also robust, since it requires minimum parameter specification: we do not need to specify the pre- and post-change detection as the classic CUSUM procedure does \cite{liu2017change}; the change detection is achieved by a simple two sample test for the running sample means over a sliding window. This result conveys a message that simple (rather than more sophisticated) change-point detection might suffice for piecewise stationary bandit. 
To the best of our knowledge, M-UCB is the first practical algorithm for  piecewise-stationary 
multi-armed bandits that uses change-point detection and whose near optimality is proved without strong parametric assumptions.

In additional, we validate numerically the scalings of the M-UCB's regret in $M$ and $K$. Experiment results in Section~\ref{sec:syn} show that the scalings are roughly $O(\sqrt{M})$ and $O(\sqrt{K})$, which suggests that our $O(\sqrt{MKT\log T})$  regret bound also reflects the right scalings of the M-UCB's regret in $M$ and $K$. 
Finally,
we compare M-UCB with state-of-the-art algorithms in numerical experiments based on a public Yahoo! dataset (Section~\ref{sec:yahoo}). In both experiments, M-UCB achieves at least $50 \%$ regret reduction with respect to the best performing state-of-the-art algorithm. 

The remainder of the paper is organized as follows: we briefly review the relevant literature in Section~\ref{sec:literature}, then we describe the piecewise-stationary bandit model in Section~\ref{sec:setting}. We discuss how to perform change-detection in the considered scenario in Section~\ref{sec:change-detection}, and motivate and propose M-UCB algorithm in Section~\ref{sec:algorithms}. We prove the regret bound in Section~\ref{sec:theory} and demonstrate experiment results in Section~\ref{sec:experiment}. 
We conclude the paper in Section~\ref{sec:conclusion}.





\section{Literature Review}
\label{sec:literature}

Most existing work on piecewise-stationary bandit problems are based on the idea to adapt to changes 
passively by adjusting the weights on the rewards. For instance, Discounted UCB (D-UCB) algorithm introduced in \cite{kocsis2006discounted} (see also \cite{garivier2011upper}) averages the past rewards with a discount factor, so it weighs more on the recent rewards to compute the UCB index of each arm. In \cite{garivier2011upper}, D-UCB policy has been proved to achieve an $O(K \sqrt{MT}\log T)$ regret. 
As a slight modification of D-UCB, the Sliding-Window UCB (SW-UCB) algorithm introduced in \cite{garivier2011upper} computes the UCB index based on only the most recent $w$ rewards and the regret is proved to be 
$O(K \sqrt{MT \log T})$. 
In \cite{auer2002nonstochastic}, the authors present EXP3.S algorithm which uses a regularization method to control the action switches and achieves an $O(\sqrt{MKT \log(KT)})$ regret. Using the idea in \cite{herbster1998tracking}, a similar algorithm called SHIFTBAND is  established in \cite{auer2002using}, which achieves an 
$O(\sqrt{MKT \log(T^3K/\delta)})$ regret with probability at least $1-\delta$. Finally, Rexp3 presented in \cite{besbes2014stochastic} achieves an $O((K\log KV_T)^{1/3}T^{2/3})$ regret, where $V_T$ is the total variation budget up to time $T$. 

There has also been work exploring the idea of monitoring the reward distributions by a change-detection (CD) algorithm and triggering the reset of the learning algorithm.
In contrast to the above algorithms that passively adapt to the changes, this type of algorithms actively locate the change-points and hence usually demonstrate better performance in practice. 
The Adapt-EvE algorithm \cite{hartland2007change} uses Page-Hinkley Test for change-detection and restart UCB1 algorithm once a change-point is detected. Taking a Bayesian point of view, \cite{mellor2013thompson} provides an algorithm by combining a Bayesian CD algorithm and Thompson Sampling. Combining one simple CD algorithm with any other MAB algorithm with a logarithm regret, \cite{yu2009piecewise} offers a windowed mean-shift detection (WMD) algorithm that achieves an $O(KM\log T)$ regret. However, their algorithm needs to query and observe the past rewards of some unpicked arms, which violates the bandit feedback model.
Combining classic MAB algorithm used in adversarial setting such as EXP3, in \cite{allesiardo2015exp3}, the authors present a EXP3.R algorithm which resets EXP3 algorithm if one CD algorithm detects that a sub-optimal arm becomes the optimal. The EXP3.R algorithm achieves an $O(NK\sqrt{T\log T})$ regret, where $N$ is the number of switches of the best arm during the run. Note that $N \leq M$ in general and $N=M$ in the worst case.

A recent and related work \cite{liu2017change} uses the CUSUM algorithm for change-point detection. Compared to this work, there are two major differences. First, we use a different change-point detection (CD) method rather than CUSUM. Our CD method is simpler, and does not require to specify any parameters. Consequently, our algorithm is applicable to  general piecewise-stationary bandits with bounded rewards, while \cite{liu2017change} is restricted to the special case with Bernoulli rewards.
%
%
Second, we use different analysis techniques to derive regret bounds. Leveraging renewal processes and classic metrics of change detection, a generalizable proof structure is established. It unlocks opportunities to prove similar regret bound with different CD methods, without taking much extra effort. 




\section{Problem Formulation}
\label{sec:setting}

\vspace{-0.1in}
\subsection{Piecewise-Stationary Bandit}
\vspace{-0.1in}

A piecewise-stationary bandit is characterized by a triple $ \left( \mathcal{K}, \mathcal{T}, \{ f_{k,t} \}_{k \in \mathcal{K}, t \in \mathcal{T}} \right)$, where $\mathcal{K} = \{1, \ldots, K\}$ is a set of $K$ arms, $\mathcal{T}= \left \{ 1,\ldots, T \right \}$ is a sequence of $T$ time steps,
and $f_{k,t}$ is the reward distribution of arm
$k$ at time $t$. Assume that arm $k$'s reward at time $t$, $X_{k,t}$, is independently drawn from $f_{k,t}$, both across arms and across time steps. 
Without loss of generality, assume that the support of $f_{k,t}$ is a subset of $ [0,1]$ for all $k \in \mathcal{K}$, $t \in \mathcal{T}$. 

We define $M$, the number of piecewise-stationary segments in the reward process to be
\begin{equation}
M = 1 + \textstyle\sum_{t=1}^{T-1} \mathbb{I}\{ f_{k,t} \neq f_{k,t+1} ~\mbox{for some}~ k \in\mathcal{K} \},
\label{num_segments}
\end{equation}
where $\mathbb{I}\{\cdot\}$ is the indicator function. Notice that by definition,  the number of change-points is $M-1$. We use
$\nu_1, \nu_2, \ldots, \nu_{M-1}$ to denote those $M-1$ change-points, and define $\nu_0=0$ and $\nu_M=T$ to simplify the exposition. 
To emphasize the ``piecewise stationary" nature of this problem, for each stationary segment $i=1,2,\ldots, M$ with $t \in [\nu_{i-1}+1, \nu_i]$, we use $f_k^i$ and $\mu_k^i$ to respectively denote the reward distribution and the expected reward of arm $k$ on the $i$th segment. Define a vector that contains all expected rewards for the $i$th segment $\mu^i = (\mu_1^i, \ldots, \mu_K^i)^\top$, $i = 1, \ldots, M$. Note that our model allows asynchronous changes to happen at  arms, i.e., the changes do not have to happen at the same time cross multiple arms. Also note that the piecewise stationary bandit model is more general than both the stochastic and the adversarial bandit models. The stochastic bandit model can be viewed as a special case of our model with $M=1$, and the adversarial bandit model can also be viewed as a special case of our model with $M=T$. 

A learning agent will repeatedly interact with this piecewise stationary bandit for $T$ times. The agent knows $\mathcal{T}$ and $\mathcal{K}$, 
but does not know $\{ f_{k,t} \}_{k \in \mathcal{K}, t \in \mathcal{T}}$ or any of its statistics such as $M$ and $\mu^i$'s. At each time step $t \in \mathcal{T}$, the  agent  chooses an action $A_t$ based on its past actions and observations, and will receive and observe the reward $X_{A_t, t}$.
%

\vspace{-0.1in}
\subsection{Regret Minimization}
\vspace{-0.1in}
The agent's objective is to maximize its expected cumulative reward in the $T$ time steps, i.e. $\max \mathbb{E}[ \sum_{t=1}^T X_{A_t,t} ] $, which is equivalent to minimize its $T$-step {\it cumulative regret} $\mathcal{R}(T)$ defined as
\begin{equation}
\mathcal{R}(T) = \textstyle\sum_{t=1}^T \max_{k \in \mathcal{K}} \mathbb{E} \left[X_{k,t} \right] - \mathbb{E}\left[ \sum_{t=1}^T X_{A_t,t} \right].
\label{regretdef}
\end{equation}
Note that the regret metric defined in (\ref{regretdef}) is stricter than the regret metric considered in most adversarial bandit papers, which is defined as
\begin{equation}
\widetilde{\mathcal{R}}(T) =\textstyle  \max_{k\in \mathcal{K}} \sum_{t=1}^T \mathbb{E} \left[ X_{k,t} \right] - \mathbb{E}\left[ \sum_{t=1}^T X_{A_t,t} \right].
\label{regretadv}
\end{equation}
Clearly $\mathcal{R}(T) \geq \widetilde{\mathcal{R}}(T)$, since the regret defined in (\ref{regretdef}) is measured with respect to the optimal piecewise stationary policy, while the regret defined in (\ref{regretadv}) is measured with respect to the optimal action in hindsight.

\vspace{-0.1in}
\subsection{Sequential Change-Point Detection}
\label{sec:change-detection}
\vspace{-0.1in}


Sequential change-point detection, which is rooted in classical statistical sequential analysis \cite{siegmund1985sequential, basseville1993detection}, aims to detect the change in underlying distributions of a sequence of observations as quickly as possible. Commonly used methods for change-point detection include CUSUM and the generalized likelihood ratio (GLR) procedure \cite{page1954continuous, willsky1976generalized}. However, for piecewise-stationary bandits, both pre-change and post-change distributions are unknown, and thus CUSUM is not suitable since it requires specifying both pre- and post-change distribution parameters. GLR can allow for unknown parameters (e.g. \cite{lai2010sequential}),  however, it is non-recursive and thus computationally expensive and not suitable for online implementation, especially for the high-dimensional setting.

Thus, we are not going to use CUSUM or GLR, but rather a simple change-point detection component based on comparing running sample means over a sliding window, as presented in Algorithm \ref{alg:cd}. 
%
%
This is computationally efficient and robust, since it has minimum parameter specification. We will show this is sufficient to guide bandit decisions as it achieves a nearly optimal regret bound.
 
\begin{algorithm}
\caption{Change detection: CD($w,b,Y_1,\ldots, Y_w$)}
\label{alg:cd}
\begin{algorithmic}[1]

\REQUIRE An even number $w$, $w$ observations $Y_1, \ldots, Y_w$ and a prescribed threshold $b>0$ 
\IF{$|\sum_{i=w/2+1}^{w} Y_i - \sum_{i=1}^{w/2} Y_i|>b$} 
\STATE Return True
\ELSE 
\STATE Return False
\ENDIF

\end{algorithmic}
\end{algorithm}

\section{M-UCB Algorithm}
\label{sec:algorithms}

Now we present the Monitored UCB (M-UCB) algorithm (as described in Algorithm~\ref{alg:bandit}) using a simple change-point detection component for the piecewise-stationary bandits. On a high level, M-UCB combines three ideas: (1) \emph{uniform sampling exploration} to ensure that sufficient data are gathered for all arms to perform CD, (2) \emph{UCB-based exploration} to learn the optimal arm on each segment, and (3) a simple change-point detection component Algorithm~\ref{alg:cd} that monitors changes and triggers exploration.


\begin{algorithm}
\caption{Monitored UCB (M-UCB)}
\label{alg:bandit}
\begin{algorithmic}[1]

\REQUIRE $T$, $K$, even integer $w>0$, $b>0$ and $\gamma \in [0, 1]$

\STATE \textbf{Initialization: }$\tau \leftarrow 0$ and $n_k \leftarrow 0$ $\forall k \in \mathcal{K}$
\FORALL{$t=1,2,\ldots, T$}
\STATE $A  \leftarrow (t-\tau) \mbox{ mod } \floor{K/\gamma} $.
\IF{$A\leq K$ }
\STATE $A_t \leftarrow A$.
\ELSE
{
\FORALL{$k = 1, \ldots, K$}
\STATE $\mbox{UCB}_k \leftarrow \frac{1}{n_k} \sum_{n=1}^{n_k} Z_{k, n}
+ \sqrt{\frac{2\log (t-\tau)}{n_k}}$.
\ENDFOR
\STATE $A_t \leftarrow \mathop{argmax}_{k\in \mathcal{K}} \mbox{UCB}_k$.
} 
\ENDIF
\STATE \mbox{Play arm } $A_t$ \mbox{and receive the reward } $X_{A_t,t}$. 
\STATE $n_{A_t} \leftarrow n_{A_t} + 1; Z_{A_t, n_{A_t}} \leftarrow X_{A_t,t}.$
\IF{$n_{A_t} \geq w$}
\IF{CD($w,b,Z_{A_t,n_{A_t}-w+1}, \ldots, Z_{A_t, n_{A_t}}$) = True}
\STATE $\tau \leftarrow t$ and $n_k \leftarrow 0 \; \forall k\in \mathcal{K}$. 
\ENDIF
\ENDIF
%
\ENDFOR
\end{algorithmic}
\end{algorithm}
Additional explanations to Algorithm~\ref{alg:bandit} are as follows. The inputs include the time horizon $T$, the number of arms $K$, and three tuning parameters $w$, $b$, and $\gamma$. Here, $w$ and $b$ are tuning parameters for the CD algorithm (line 15), which control the power of the CD algorithm; and $\gamma$ controls the fraction of the uniform sampling (line 3). Let $\tau$ denote  the last detection time, and let $n_k$  denote the number of observations from the $k$th arm after $\tau$.  In Remark~\ref{rem:w} of the subsequent section, we discuss how to choose these parameters based on our theoretical analysis.

At each time $t$,  M-UCB proceeds as follows. First, M-UCB determines whether to perform a uniform sampling exploration or a UCB-based exploration at each time, according to conditions given in line 3 and 4 to ensure that the fraction of time steps performing uniform sampling is roughly $\gamma$. If UCB-based exploration is used at time $t$, then the standard UCB1 indices are computed using observations from the last detection time $\tau$ to the current time, and an action is chosen greedily based on the UCB1 indices (line 7-10). 
By playing the chosen arm, we observe the reward and update some statistics (line 12-13). Finally, when at least $w$ observations for the chosen arm have been gathered after the last detection time $\tau$,  M-UCB will perform CD via Algorithm~\ref{alg:cd} and restarts exploration if necessary.

We would like to emphasize that the uniform sampling exploration is crucial to M-UCB. This is because that the standard UCB exploration tends to select very infrequently the arms which it believes to be suboptimal. Thus, standard UCB exploration cannot quickly detect changes in these ``infrequently visited'' arms. 

\section{Near Optimality of M-UCB}
\label{sec:theory}
In this section, we present our main result: the M-UCB algorithm based on simple change-point detection algorithm achieves a nearly optimal regret bound.

Recall that $T$ is the time horizon, $K$ is the number of arms, $M$ is the number of piecewise-stationary segments,
$\nu_0, \ldots, \nu_M$ are the change-points, and for each $i=1,\ldots, M$, $\mu^i \in [0,1]^K$ is a vector encoding the expected rewards 
of all arms on segment $i$. 
We also use $\mathbb{P}$ and $\mathbb{E}$ to respectively denote the probability measure and the expectation according to the piecewise-stationary bandit characterized by the tuple $(T,K,M, \{\nu_i\}_{i=0}^{M}, \{\mu^i\}_{i=1}^M)$. 
%
%
To simplify the exposition, we define the ``sub-optimal gap" of arm $k$ on the $i$-th piecewise-stationary segment
as
\begin{equation}
\Delta_k^{(i)} = \textstyle \max_{\tilde{k} \in \mathcal{K}} \{ \mu_{\tilde{k}}^i \} - \mu_k^i \quad \forall 1\leq i\leq M, \, k \in \mathcal{K},
\label{eq:Delta}
\end{equation}
and the amplitude of the change of arm $k$ at the $i$th change-point as
\begin{equation}
\delta_k^{(i)} = |\mu_k^{i+1} - \mu_k^{i}|, \quad \forall 1\leq i\leq M-1, \,  k \in \mathcal{K}.
\label{eq:delta}
\end{equation}
Moreover, recall that $w$, $b$ and $\gamma$ are the tuning parameters for Algorithm~\ref{alg:bandit}.
We define $L = w\ceil{K/\gamma}$ for shorthanded notation. 

We make the following assumptions for our theoretical analysis:
\begin{assu}
\label{assu0}
The learning agent can choose $w$ and $\gamma$ s.t. (a) $M<\floor{T/L}$ and $\nu_{i+1}-\nu_i > L, \forall 0\leq i\leq M-1$, 
and
(b) $ \forall 1\leq i\leq M-1$, 
$\exists k \in \mathcal{K}$ s.t.
$ \delta_k^{(i)} \geq 2\sqrt{\log(2KT^2)/w}+2\sqrt{\log(2T)/w} $.
\end{assu}

We would like to clarify that Assumption~\ref{assu0} is only required for the analysis; Our proposed M-UCB algorithm (Algorithm~\ref{alg:bandit}) can be implemented regardless of this assumption. As is shown in Section~\ref{sec:experiment}, in the real-world experiments, our algorithm works well even if Assumption~\ref{assu0} does not hold. Notice that relevant literature, such as \cite{liu2017change}, makes similar assumptions. Moreover, compared with \cite{liu2017change}, we have relaxed a major assumption: they also assume the rewards are Bernoulli, which is not assumed in our algorithm and analysis.

We now briefly motivate and explain Assumption~\ref{assu0}. Intuitively, Assumption~\ref{assu0}(a) means that the length of the time interval between two consecutive change-points is larger than $L$. This guarantees that Algorithm \ref{alg:bandit} can select at least $w$ samples from every arm, and these samples are used to feed the CD algorithm. Assumption~\ref{assu0}(b) means that the change amplitude is over certain threshold for at least one arm at each change-point. This guarantees that the CD algorithm is able to detect the change quickly with limited information. 
If a lower bound $\delta>0$ on $\min_i \max_{k \in \mathcal{K}} \delta_k^{(i)}$ can be assumed, then one can choose
\begin{equation}
\label{eqn:w_choice}
w \approx (4/\delta^2) \cdot[(\log(2KT^2))^{1/2}+(\log(2T))^{1/2}]^2
\end{equation}
 to satisfy Assumption~\ref{assu0}(b).
Our main result is the following regret bound:


\begin{theorem}
\label{thm:bandit}
Running Algorithm~\ref{alg:bandit} with $w$ and $\gamma$ satisfying
Assumption~\ref{assu0} and $b=[w\log(2KT^2)/2]^{1/2}$,  we have
\begin{equation}
\begin{split}
\mathcal{R}(T)
\leq 
& 
\underbrace{\textstyle \sum_{i=1}^M \tilde{C}_i}_{(a)} +\underbrace{\gamma T}_{(b)} \\
+& \underbrace{ \textstyle \sum_{i=1}^{M-1}\frac{2K \cdot \min(\frac{w}{2},\ceil{\frac{b}{\delta^{(i)}}}+3\sqrt{w})}{\gamma}}_{(c)}+\underbrace{3M}_{(d)},
\end{split}
\label{eq:regret1}
\end{equation}
where $\delta^{(i)} = \max_{k \in \mathcal{K}}\delta_k^{(i)}$ and 
$
\tilde{C}_i = 8\textstyle\sum_{\Delta_k^{(i)}>0} \frac{\log T}{\Delta_k^{(i)}} + \left(1+\frac{\pi^2}{3}+K\right) \sum_{k=1}^K \Delta_k^{(i)}$.
\end{theorem}
Theorem \ref{thm:bandit} reveals that the regret incurred by M-UCB  can be decomposed into four terms.
Terms (a) and (b) in equation~(\ref{eq:regret1}) bound on the exploration costs: term (a) bounds the cost of the UCB-based exploration,
and term (b) bounds the cost of the uniform sampling. On the other hand, terms (c) and (d) bound the change-point detection costs: term (c) bounds the cost associated with the detection delay of the CD algorithm, and term (d) is incurred by the unsuccessful and incorrect detections of the change-points. The following corollary follows immediately from Theorem~\ref{thm:bandit}.



\begin{corollary}
\label{corr:regret}
Assume $\delta>0$ is a lower bound on $\min_i \max_{k \in \mathcal{K}} \delta_k^{(i)}$.
If we run Algorithm~\ref{alg:bandit} with a window-length $w$, 
\[b=[w\log(2KT^2)/2]^{1/2},\] and
\[\gamma = \sqrt{\textstyle (M-1) K\cdot\min(w/2,\ceil{b/\delta}+3\sqrt{w})/(2T)},\]
then we have
\begin{multline}
\mathcal{R}(T)
\leq \textstyle  \sum_{i=1}^M \tilde{C}_i \\
+ 4 \sqrt{ (M-1) TK\cdot \min(w/2,\ceil{b/\delta}+3\sqrt{w})}+3M.
\label{eq:regret2}
\end{multline}
\end{corollary}

For any fixed $w$, the upper bound for the regret in equation~(\ref{eq:regret2}) is 
$$O(\sqrt{MKT\log T})=\tilde{O} (\sqrt{MKT}),$$
 where $\tilde{O}$ notation hides logarithmic factors. 
Compared with the lower bound in $\Omega(\sqrt{T})$ \cite{garivier2008upper}, our regret bound is asymptotically tight up to a logarithmic factor. 
In Section~\ref{sec:syn}, 
we  validate numerically that when scaled by $1/\sqrt{T}$, the scalings of Algorithm~\ref{alg:bandit}'s regret in $M$ and $K$ are roughly $O(\sqrt{M})$ and $O(\sqrt{K})$, as is suggested in Corollary~\ref{corr:regret}. We leave the derivation of the lower bound in $K$ and $M$ to future work.


Corollary~\ref{corr:regret} also sheds some  insights on how to choose tuning parameters $w$, $b$, and $\gamma$ in Algorithm~\ref{alg:bandit}. 
\begin{remark}[Algorithm Parameter Tuning]
\label{rem:w}
We now discuss how to choose algorithm parameters $w$, $b$, and $\gamma$ based on Corollary~\ref{corr:regret}.
In practice, we mainly care about large changes since small changes do not incur much regret. 
Assume an minimum change size $\tilde{\delta}>0$ then following from equation~\eqref{eqn:w_choice} and Corollary~\ref{corr:regret}, we can choose the window size $w \approx (4/\tilde{\delta}^2) \cdot[(\log(2KT^2))^{1/2}+(\log(2T))^{1/2}]^2$,
$b \approx [w\log(2KT^2)/2]^{1/2}$, and $\gamma \approx (\textstyle\sum_{i=1}^{M-1}K\cdot\min(w/2,\ceil{b/\tilde{\delta}}+3\sqrt{w})/(2T)^{1/2}$.
\end{remark}

The proof outline for Theorem~\ref{thm:bandit} is provided in section \ref{sec:proof_outline} and more details for technical lemmas are given in Appendix~\ref{sec:proofs}. The main steps of the proof are as follows. First, we rely on standard bandit analysis to decompose $\mathcal{R}(T)$ over a set of ``good" events and a set of ``bad" events: the good events include all the sample paths that Algorithm \ref{alg:bandit} reinitializes the UCB algorithm quickly after any change-point. The set of bad events includes all the sample paths that Algorithm \ref{alg:bandit} that either fails to reinitialize the UCB algorithm quickly when there is a change-point or incorrectly reinitializes the UCB algorithm when there is not any change-point. This enables us to couple the change-point detection analysis with bandit analysis, and we identify that the parameters specified in Theorem~\ref{thm:bandit} will ensure that the set of good events occurs with a high probability.

\subsection{Proof Outline of Theorem~\ref{thm:bandit}}
\label{sec:proof_outline}

In this subsection, we outline the proof for Thereom~\ref{thm:bandit}. Detailed proofs are provided in Appendix~\ref{sec:proofs}.
%
%

First, we bound the regret incurred by Algorithm \ref{alg:bandit} in the stationary scenario with $M=1$, $\nu_0 = 0$, and $\nu_1 = T$. 

\begin{lemma}[Regret bound for the M-UCB algorithm in stationary scenarios]
\label{lem:stationary_bandit}
Consider a stationary scenario with $M=1$, $\nu_0 = 0$, and $\nu_1 = T$. Under Algorithm \ref{alg:bandit} with parameter $w,b$ and $\gamma$, we have that
\begin{equation}
\mathcal{R}(T) \leq T\cdot \mathbb{P}(\tau_1 \leq T) + \tilde{C}+ \gamma T,
\label{eq:stationary_bandit}
\end{equation}
where $\tau_1$ is the first detection time and 
\[
\tilde{C} = 8\textstyle\sum_{\Delta_k^{(1)}>0} \frac{\log T}{\Delta_k^{(1)}} + \left(1+\frac{\pi^2}{3}+K\right) \sum_{k=1}^K \Delta_k^{(1)}.
\]

\end{lemma}

\begin{remark}
Lemma \ref{lem:stationary_bandit} shows that the regret for the M-UCB algorithm in the stationary scenario is incurred by three sources. The term $\mathbb{P}(\tau_1\leq T)$ on the right-hand side of (\ref{eq:stationary_bandit}) is the probability of raising one false alarm. This can be controlled to be small through setting appropriate algorithm parameters. The term $\tilde{C}$ is the classic regret bound for the UCB-based exploration in stationary scenarios. The term $\gamma T$ is incurred by the uniform sampling exploration. 
\end{remark}

Please refer to Appendix~\ref{sec:proof_stationary_bandit} for the proof of Lemma~\ref{lem:stationary_bandit}.
Next, we bound the probability of restarting Algorithm \ref{alg:bandit} 
when there is no change-point. This probability is equivalent to the probability of the CD algorithm raising a false alarm in the stationary scenario discussed above.

\begin{lemma}[Probability of raising false alarms in the stationary scenario]
\label{lem:falsealarm}
Consider a stationary scenario with $M=1$. Then under Algorithm \ref{alg:bandit} with parameter $w<T$, $b$ and $\gamma$, 
we have that
\[
\mathbb{P}(\tau_1\leq T) \leq wK\left(1-\left(1-2 \exp \left(-2b^2/w\right)\right)^{\floor{T/w}} \right),
\]
where $\tau_1$ is the first detection time. 

\end{lemma}

\begin{remark}
Using the fact that $(1-x)^a >1-ax$ for any $a>1$ and $0<x<1$, we have that in Lemma \ref{lem:falsealarm} 
$
\mathbb{P}(\tau_1\leq T) \leq 2KT\exp(-2b^2/w). 
$
Therefore, setting $b=[w\log(2KT^2)/2]^{1/2}$ we have that $\mathbb{P}(\tau_1\leq T) \leq 1/T$, which means that it is expected to raise at most only one false alarm in a stationary scenario with $T$ time steps. 
\end{remark}

Please refer to Appendix~\ref{sec:proof_falsealarm} for the proof of Lemma~\ref{lem:falsealarm}.
Then, we establish a lower bound on the probability that the CD algorithm (Algorithm \ref{alg:cd}) achieves a successful detection in scenarios with one change-point, i.e. $M=2$.


\begin{lemma}[Probability of achieving a successful detection with $M=2$]
\label{lem:detection}
Consider a piecewise-stationary scenario with $M=2$, and recall that $L=w\ceil{K/\gamma}$. Assume that $\nu_2-\nu_1>L/2$. For any $\mu^1, \mu^2 \in [0,1]^K$ satisfying 
$$\delta_{\tilde{k}}^{(1)} \geq 2b/w+c$$
 for some $\tilde{k}\in \mathcal{K}$ and $c>0$, under Algorithm~\ref{alg:bandit}, we have that 
\[
\mathbb{P}(\nu_1<\tau_1\leq \nu_1+L/2 \mid \tau_1 > \nu_1 ) \geq 1-2\exp\left(-wc^2/4\right).
\]

\end{lemma}

\begin{remark}
Setting $b=\sqrt{w\log(2KT^2)/2}$ and $c= 2\sqrt{\log(2T)/w}$ in Lemma \ref{lem:detection}, we have that with probability at least $1-1/T$, a change can be detected in $L/2$ steps after the change occurs , provided that $\delta_{\tilde{k}}^{(1)}$ is greater than $C_1/\sqrt{w}$ for some constant $C_1$ that only depends on $T$ and $K$. This shows that we can set a larger $w$ to achieve a successful detection of a smaller change. 
\end{remark}

Please refer to Appendix~\ref{sec:proof_detection} for the proof of Lemma \ref{lem:detection}. Lemma \ref{lem:detection} relates the tuning parameter $b$ and $w$ to the smallest change that the CD algorithm can successfully detect with high probability. 

In the next lemma, for scenarios with $M=2$, we bound the expected detection delay (EDD) by a function of the change amplitude, given that the change can be detected successfully. In other words, Lemma~\ref{lem:detection} characterizes a lower bound on change amplitude to ensure that the detection delay is no more than $L/2$ with high probability. When the change amplitude is not so small, the EDD can be smaller than $L/2$, as presented in the following lemma. 

\begin{lemma}[Expected detection delay]
\label{lem:DD}
Consider a piecewise-stationary scenario with $M=2$, and recall that $L=w\ceil{K/\gamma}$. 
Assume that $\nu_2-\nu_1>L/2$. For any $\mu^1, \mu^2 \in [0,1]^K$ satisfying $\delta_{\tilde{k}}^{(1)} > 2b/w+c$ for some $\tilde{k}\in \mathcal{K}$, we have that 
\begin{equation*}
\begin{split}
&\mathbb{E}[\tau_1 - \nu_1 \mid \nu_1<\tau_1\leq \nu_1+L/2 ] \\
& \leq \frac{\min(L/2, \ceil{b/\delta_{\tilde{k}}^{(1)}} + 3\sqrt{w} \cdot \ceil{K/\gamma})}{1-2\exp\left(-wc^2/4\right)}.
\end{split}\nonumber
\end{equation*}
\end{lemma}

Please refer to Appendix~\ref{sec:proof_dd} for the proof of Lemma~\ref{lem:DD}.

Theorem~\ref{thm:bandit} can be proved based on the above four lemmas and properties of renewal processes. Specifically, we decompose $\mathcal{R}(T)$ over a set of ``good" events and a set of ``bad" events: the good events include all the sample paths that Algorithm \ref{alg:bandit} reinitializes the UCB algorithm correctly and quickly after all change-points. The set of bad events includes all the sample paths that Algorithm \ref{alg:bandit} that either fails to reinitialize the UCB algorithm quickly when there is a change-point (large detection delay) or incorrectly reinitializes the UCB algorithm when there is not any change-point (false alarm). Lemma~\ref{lem:falsealarm} and~\ref{lem:detection} can be used to upper bound the probabilities of the bad events. Together with the naive bound $\mathcal{R}(T) \leq T$, we can bound the regret in the bad events. On the other hand, Lemma~\ref{lem:stationary_bandit} and~\ref{lem:DD} can be used to upper bound the regret in the good events.
Please refer to Appendix~\ref{sec:proof_main} for the detailed proof for Theorem~\ref{thm:bandit}.

\vspace{-0.1in}
\section{Experiments}
\label{sec:experiment}

In this section, we present some numerical experiments  to validate the performance of M-UCB. We first verify the scalings of  M-UCB's regret in $M$ and $K$ and then compare M-UCB with state-of-the-art algorithms on
a publicly available benchmark Yahoo! dataset.

\subsection{Regret Scalings in $M$ and $K$}
\label{sec:syn}


To eliminate the scaling issue caused by different $T$'s, in this subsection we scale the empirical regret by $1/\sqrt{T}$. For the illustrative purposes, we assume the rewards are Bernoulli distributed. 

We first show the regret scaling in $M$. 
We fix $K=10$, and let
the locations of change-points to be evenly spaced with interval of length $20000$ so for any $M$ we have $T=20000\cdot M$. For the reward sequence of each arm, we set $\mu^{(i)} = \mu$ when $i$ is odd and $\mu^{(i)}=1-\mu$ when $i$ is even, where $\mu \in [0,1]^K$ is randomly chosen such that the difference between the largest and smallest entry is larger than $0.6$.
Consider an upper bound $20000\times 25$ for $T$, an upper bound $10$ for $K$ and a lower bound $0.6$ for $\delta$. Based on Remark \ref{rem:w}, we can set $w=800$ and  $b=\sqrt{(w/2)\cdot \log(2KT^2)}$ and $\gamma = \sqrt{(M-1)K\cdot (2b+3\sqrt{w})/(2T)}$. 
For each $M$, we randomly generate $100$ instances, and for each instance, we run the M-UCB algorithm for $50$ times. We generate the averaged regret by averaging over these $5000$ simulations. 
The results are shown in Figure \ref{fig:R_versus_M}, which can be fitted using the simple model $y=c+ax^b$ to obtain an estimated order $b=0.55$,
which means that the regret is roughly on the order of $O(\sqrt{M})$.

\begin{figure}[h]
  \centering
    \includegraphics[width=0.7\linewidth]{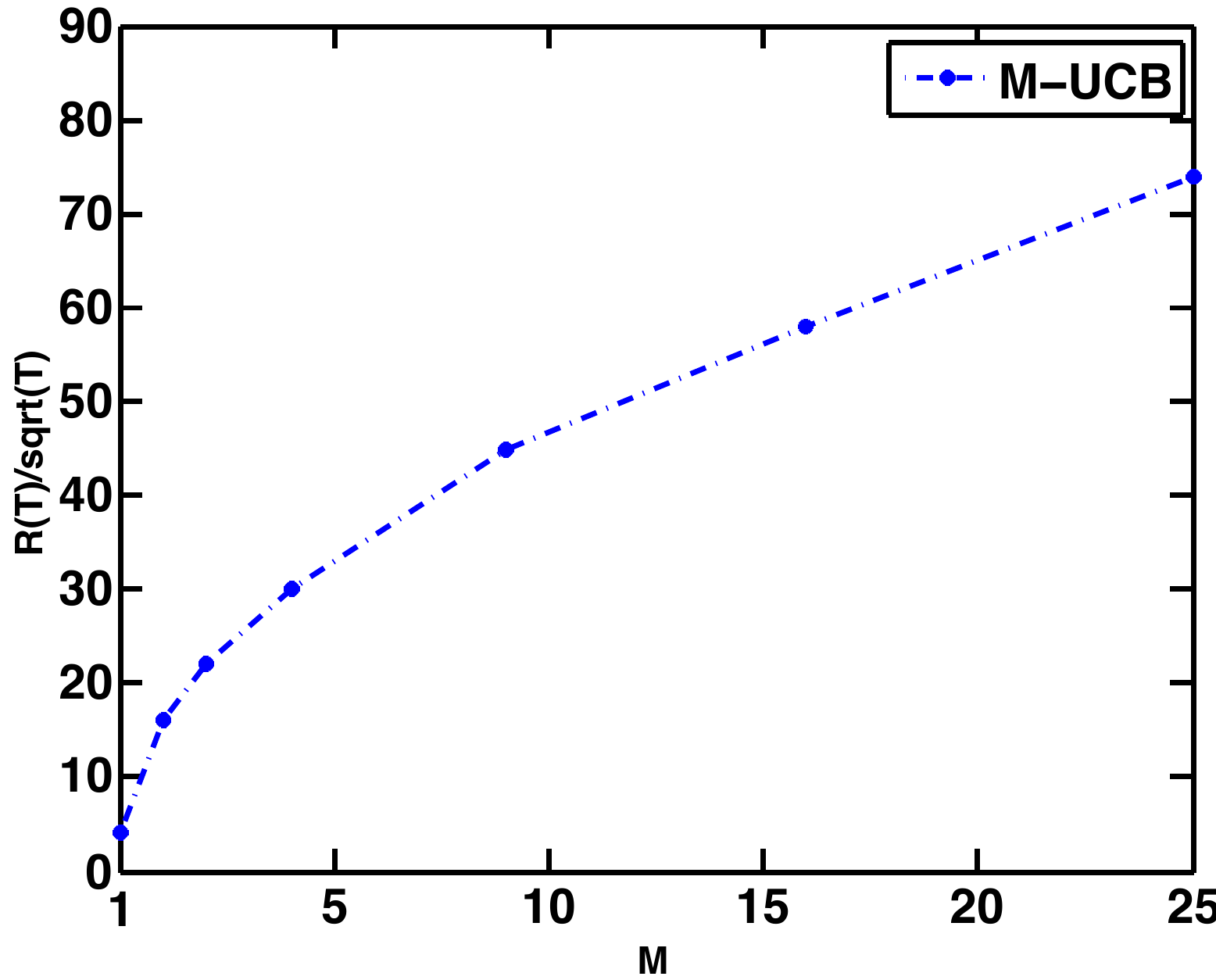}
  \caption{Cumulative regret of M-UCB up to time $T$ scaled by $\sqrt{T}$ versus $M$.}
\label{fig:R_versus_M}
\end{figure}

We then demonstrate the regret scaling in $K$. We fix $M=4$, $T=3\times 10^5$ and let the change-points to be evenly spaced. For each $i=1,\ldots,4$, we randomly generate $\mu^{(i)}\in [0,1]^K$ such that the difference between the largest and smallest entry is larger than $0.6$ and one combination of $K$ and $(\mu^{(i)})_{i=1}^4$ forms one instance.  We set the same algorithmic parameters as those in the first simulation example. We then  generate $100$ random instances, and for each instance we repeat our algorithm for $50$ times to obtain the averaged regret. 
The results are shown in Figure \ref{fig:R_versus_K}, which again can be fitted with the simple model $y=c+ax^b$ to obtain an estimated order $b=0.53$,
which means that the regret grows roughly at a rate of $O(\sqrt{K})$. 

\begin{figure}[h]
  \centering
    \includegraphics[width=0.7\linewidth]{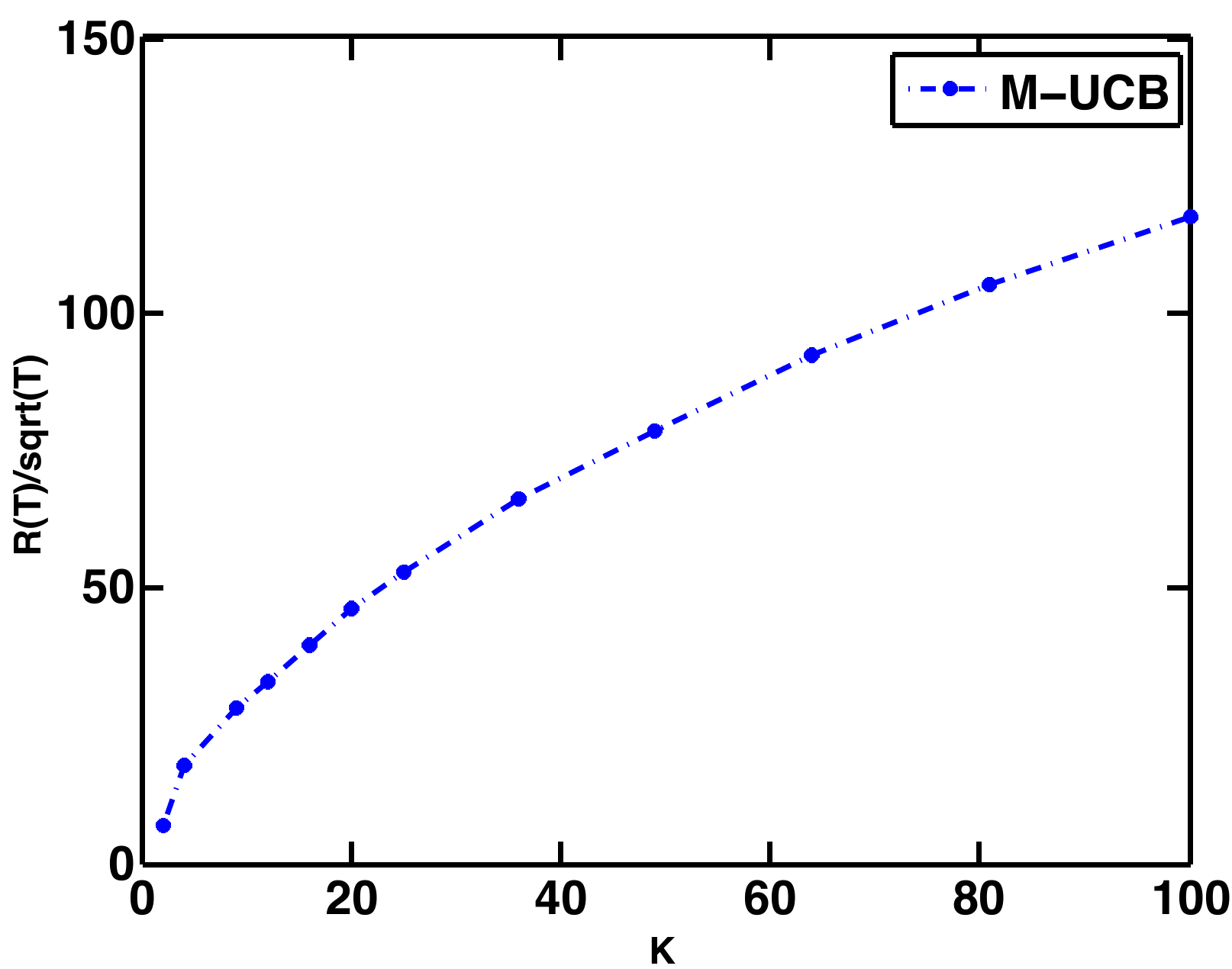}
  \caption{Cumulative regret of M-UCB up to time $T$ scaled by $\sqrt{T}$ versus $K$.}
\label{fig:R_versus_K}
\end{figure}

The above results suggest that the regret bound in Corollary~\ref{corr:regret} is roughly tight in $M$ and $K$ for M-UCB algorithm.



%

\subsection{Experiment on Yahoo! Dataset}
\label{sec:yahoo}

We compare the expected cumulative regret of different algorithms using the benchmark dataset publicly published by Yahoo!\footnote{Yahoo! Front Page Today Module User Click Log Dataset on https://webscope.sandbox.yahoo.com}. This dataset
provides a binary value for each arrival to represent whether the user clicks the specified article \cite{chu2009case, li2011unbiased}. We use one arm to represent one article and assume a Bernoulli reward (one if the user clicks the article and zero otherwise). The goal is set to maximize the expected number of clicked articles using strategies that select one article for each arrival sequentially. We randomly select six different articles of which the click-through rates are greater than zero within one five-day horizon, where the click-through rates are computed by taking the mean of the number of times each article being clicked every $43200$ seconds (which corresponds to a half day). If the difference between the estimated click-through rate of the current half day and that of the last half day is less than $0.01$, we set the click-through rate of the current half day as that of the last half day. In this way, we obtain a piecewise-stationary scenario with $T=43200\times 10 = 4.32\times 10^5$, $K=6$ and $M=9$, as shown in Figure \ref{fig:Yahoo_Arm}. 

\begin{figure}[h]
\centering
\begin{minipage}{.4\textwidth}
  \centering
  \includegraphics[width=.8\linewidth]{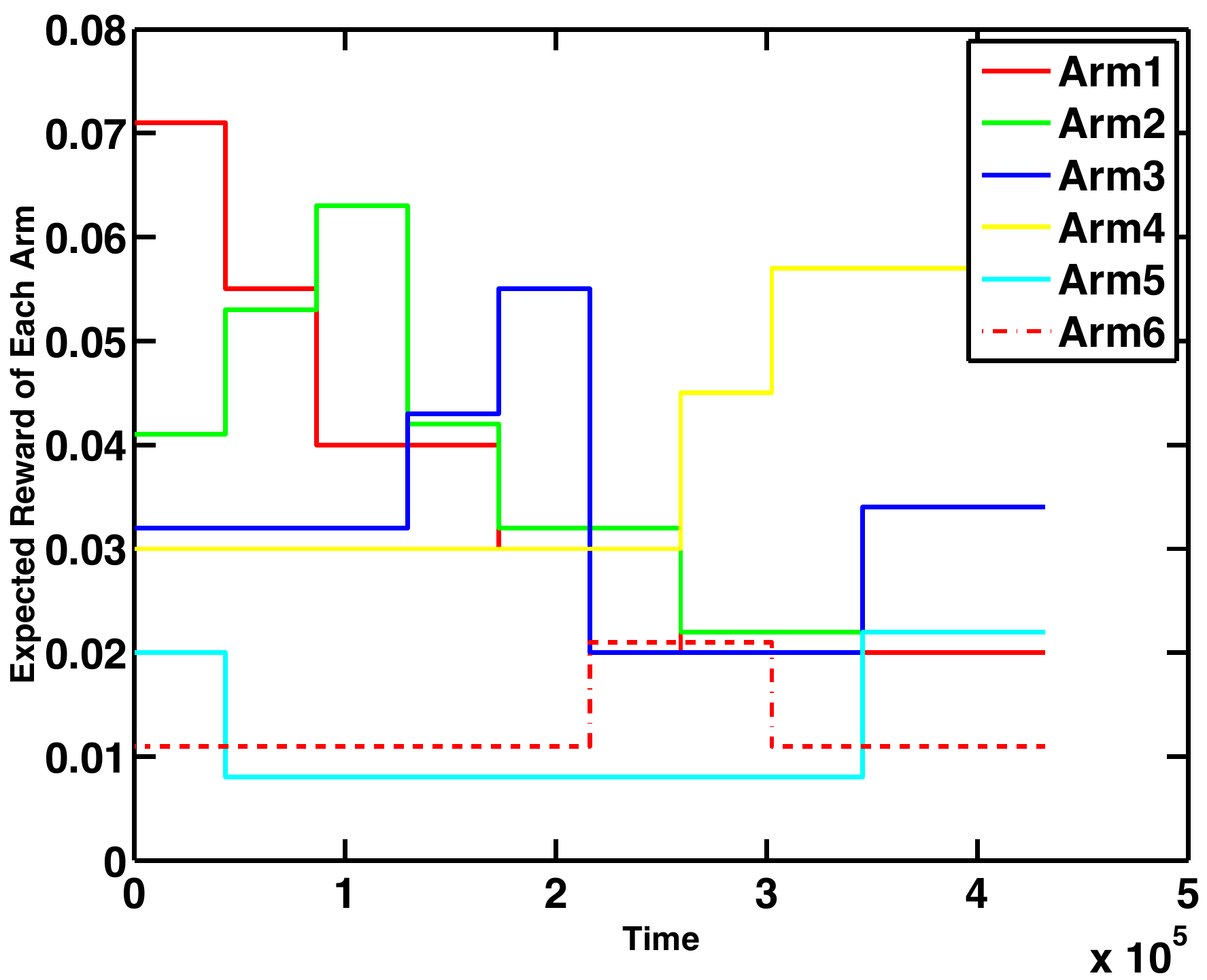}
  \captionof{figure}{Click-through rates computed from Yahoo! dataset with $T=4.32\times 10^5, K=6$ and $M=9$.}
  \label{fig:Yahoo_Arm}
\end{minipage}%
\hspace{0.47in}
\begin{minipage}{.42\textwidth}
  \centering
  \includegraphics[width=0.8\linewidth]{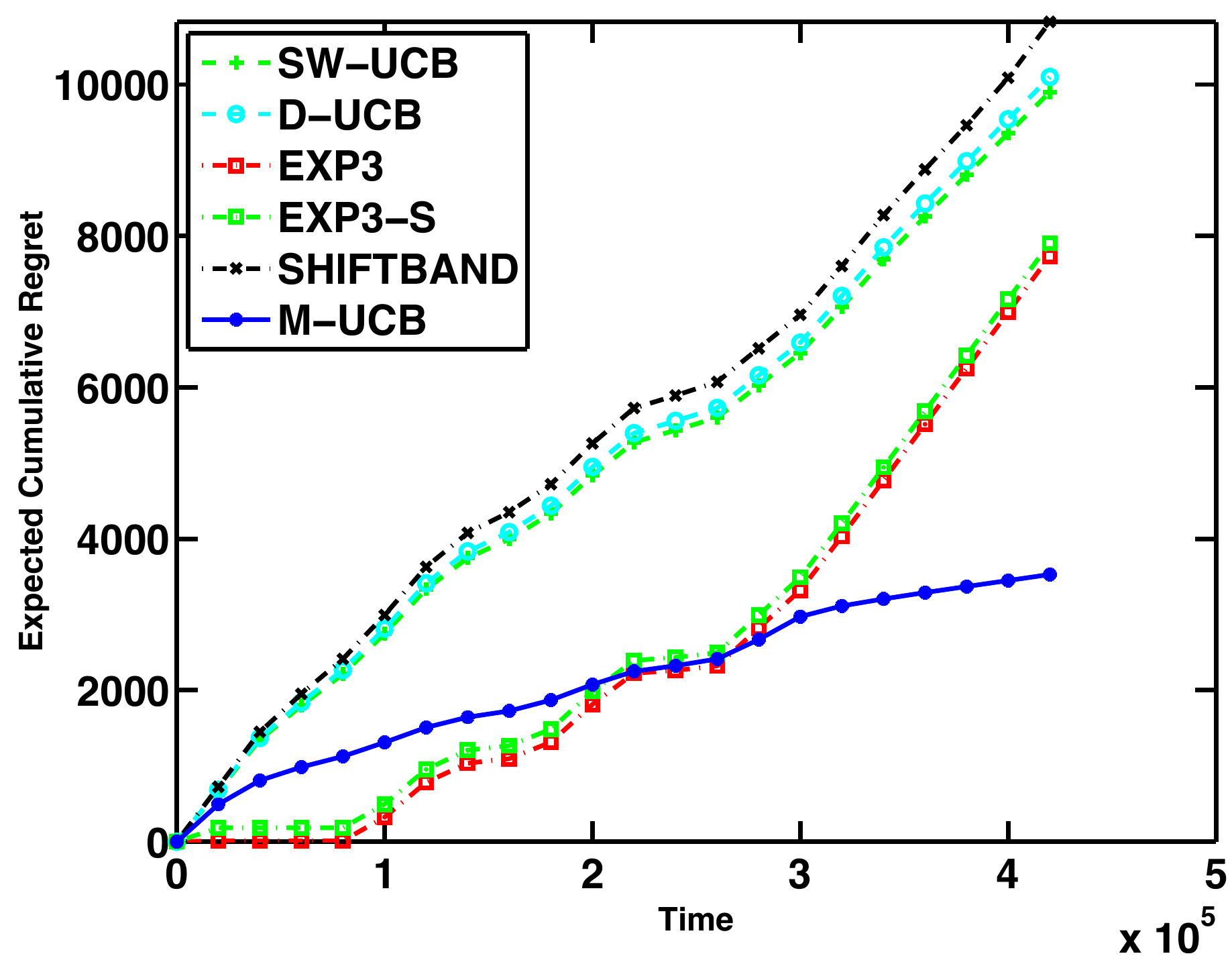}
  \captionof{figure}{Expected cumulative regrets for different algorithms under the piecewise-stationary scenario shown in Figure \ref{fig:Yahoo_Arm}.}
  \label{fig:Yahoo_Regret}
\end{minipage}
\end{figure}


Along with our algorithm using the same $w, b$ and $\gamma$ as those in subsection \ref{sec:syn}, we run five other algorithms, Discounted UCB (D-UCB), Sliding-Window UCB (SW-UCB), EXP3, EXP3.S and SHIFTBAND for comparison. Based on the theoretical results in \cite{garivier2008upper}, we choose $\gamma = 1-0.25\sqrt{(M-1)/T}$ and $\xi = 0.5$ for D-UCB and choose $\tau = 2\sqrt{T\log T/(M-1)}$ in SW-UCB. Based on Theorem $1$ in \cite{auer2002using}, we choose $\delta = 0.05$, $\alpha = 2\sqrt{\log(T^3K/\delta)}$, $\beta = 1/T$ and $\eta = \sqrt{\log(TK)M/(TK)}$ for the SHIFTBAND algorithm. For EXP3 and EXP3.S algorithms we select parameters to be the same as those in \cite{auer2002nonstochastic}. \footnote{Specifically, The parameters for EXP3 and EXP3.S are selected based on Corollary 3.2 and 8.2 in \cite{auer2002nonstochastic}.} The expected cumulative regret is computed by taking the average of the regrets for $100$ independent Monte Carlo trials, as shown in Figure \ref{fig:Yahoo_Regret}.


The results show that the M-UCB algorithm achieves a better performance than other algorithms even if all the algorithms seem to have a sub-linear regret. Compared with EXP3 and EXP3.S, M-UCB achieves a $50\%$ reduction of the cumulative regret and this number is $60\%$ if we make comparisons with SW-UCB, D-UCB and SHIFTBAND algorithms. 

It is worth clarifying that our experiment on the Yahoo dataset does not satisfy Assumption \ref{assu0} in the sense that it contains many small-magnitude changes. %
Thus, we believe it is a fair comparison for all algorithms. 
Specifically, in this case $K=6$ and $T=432000$, and we choose $w=800$ for M-UCB. Thus, Assumption~\ref{assu0} requires that all changes have magnitude no less than $0.64$. However, as is shown in Figure~\ref{fig:Yahoo_Arm}, all changes have magnitude less than $0.1$. This experiment shows that M-UCB works well and outperforms state-of-the-art baselines even if Assumption \ref{assu0} does not hold.

\section{Conclusion}
\label{sec:conclusion}

In this paper, we have developed a so-called M-UCB algorithm (Algorithm~\ref{alg:bandit}) for piecewise-stationary bandits with bounded rewards. M-UCB combines the UCB (with uniform exploration) with a simple change-point detection component based on running sample means over a sliding window. We prove that M-UCB algorithm achieves a nearly optimal regret bound on the order of $O(\sqrt{MKT\log T})$ under mild technical conditions. Our experiment results also show that it can achieve significant regret reduction with respect to the state-of-the-art algorithms in numerical experiments based on real-world datasets.

Our proposed M-UCB algorithm is based on the classical UCB1 algorithm. We may improve by considering other exploration schemes (e.g. KL-UCB, Thompson sampling) in the current setup. One can foresee that as long as the exploration schemes are statistically efficient, then a variant of our analysis will carry through.

%

\clearpage


%

\newpage
\setstretch{1}

\begingroup
\renewcommand{\section}[2]{}%
\bibliography{bandit_changepoint}
\endgroup


\newpage
\onecolumn
\appendix
\begin{center}
\textbf{\Large Appendices}
\end{center}

\section{Detailed Proofs of Theorem~\ref{thm:bandit}}
\label{sec:proofs}

\subsection{Proof of Lemma~\ref{lem:stationary_bandit}}
\label{sec:proof_stationary_bandit}

\begin{proof}[Proof of Lemma \ref{lem:stationary_bandit}]
Define that $R(s) = \sum_{t=1}^s \max_{k\in \mathcal{K}} X_{k,t} - X_{A_t, t}$, then we have that $\mathcal{R}(T) = \mathbb{E}[R(T)]$.

\begin{equation}
\begin{split}
\mathcal{R}(T) = & \mathbb{E}[R(T)] \\
= &\mathbb{E}[R(T) \mathbb{I}\{\tau_1 \leq T\}] + \mathbb{E}[R(T) \mathbb{I}\{\tau_1 > T\}] \\
\leq & T \cdot \mathbb{P}(\tau_1 \leq T) + \mathbb{E}[R(T) \mathbb{I}\{\tau_1 > T\}]. 
\end{split}
\end{equation}

Define $N_k(t)$ as the number of times arm $k$ has been selected by the Algorithm \ref{alg:bandit} in the first $t$ steps, i.e., $N_k(t)=\sum_{i=1}^t \mathbb{I}(A_i = k)$. No false alarm is raised and we do not restart the UCB algorithm if the evnet $\{\tau_1 >T\}$ happens. Therefore, we have the following equation:
\begin{equation}
\begin{split}
\mathbb{E}[R(T) \mathbb{I}\{\tau_1>T\}]  =& \sum_{\Delta_k^{(1)}>0} \Delta_k^{(1)} \cdot \mathbb{E}[N_k(T) \mathbb{I}\{\tau_1>T\}].
\end{split}\nonumber
\end{equation}
Thus, it remains to show an upper bound for $\mathbb{E}[N_k(T) \mathbb{I}\{\tau_1>T\}]$.  
By the definition of Algorithm \ref{alg:bandit}, we have that for any $k\in \mathcal{K}$
\begin{equation}
\begin{split}
&N_k(T) \mathbb{I}\{\tau_1>T\} \\
=&\sum_{t=1}^T \mathbb{I}\{A_t = k, \tau_1>T,N_k(t) < l\} \\
+& \sum_{t=1}^T \mathbb{I}\{A_t = k, \tau_1>T,N_k(t) \geq l\}  \\
\leq& l+ \sum_{t=1}^T \mathbb{I}\{t \mbox{ mod } \floor{K/\gamma} = k, N_k(t)\geq l \}  \\
+&\sum_{t=1}^T \mathbb{I}\{k = \mathop{argmax}_{\tilde{k}\in \mathcal{K}} \mbox{UCB}_{\tilde{k}}, N_k(t) \geq l \}\\
\leq & l+\ceil{T\gamma/K} + \sum_{t=1}^T \mathbb{I}\{k = \mathop{argmax}_{\tilde{k}\in \mathcal{K}} \mbox{UCB}_{\tilde{k}}, N_k(t) \geq l \},
\end{split}
\end{equation}
where the first inequality is due to the fact that if the event $\{A_t=k, \tau_1>T \}$ happens, then we do not restart the UCB algorithm before time $T$ and the selection of the $k$th arm is based on either the uniform sampling or the largest UCB index in a stochastic bandit setting. 
Setting $l=\ceil{8\log T/(\Delta_k^{(1)})^2}$ and following the same argument as in the proof of Theorem 1 of \cite{auer2002finite}, we have that 
\[
\mathbb{E}[N_k(T) \mathbb{I}\{\tau_1>T\}]\leq \frac{T\gamma}{K} + \frac{8\log T}{(\Delta_k^{(1)})^2} + 1+\frac{\pi^2}{3}+K.
\]
Summing over $k \in \mathcal{K}$ we prove the result. 

\end{proof}

\subsection{Proof of Lemma~\ref{lem:falsealarm}}
\label{sec:proof_falsealarm}

\begin{proof}[Proof of Lemma \ref{lem:falsealarm}]
Define $\tau_{k,1}$ as the first detection time of the $k$th arm. Then, $\tau_1 = \min_{k\in \mathcal{K}} \{\tau_{k,1}\}$ since Algorithm \ref{alg:bandit} is designed to reinitialize the UCB algorithm if a change is detected on any of the $K$ arms. Using the union bound, we have that 
\[
\mathbb{P}(\tau_1\leq T) \leq \sum_{k=1}^K \mathbb{P}(\tau_{k,1}\leq T).
\]
Define that for any $k \in \mathcal{K}$ and $t\geq w$ 
\begin{equation}
S_{k,t} = \left| \sum_{i=t-w/2+1}^{i=t} Z_{k,i} - \sum_{i=t-w+1}^{t-w/2} Z_{k,i} \right|.
\label{eq:stat}
\end{equation}
Then, for any $k\in \mathcal{K}$, $\tau_{k,1}$ is given by
\[
\tau_{k,1} = \inf\{t \geq w:  S_{k,t} > b\}
\]
Let $\mathbb{Z}^+$ be the set of all positive integers. Define that for any $0\leq j\leq w-1$ the stopping times
\begin{equation}
\begin{split}
&\tau_{k,1}^{(j)} = \inf \left\{t=j+nw, n\in \mathbb{Z}^+: S_{k,t} > b \right\}.
\end{split}\nonumber
\end{equation}
We have that $\tau_{k,1} = \min\{\tau_1^{(0)},\ldots, \tau_1^{(w-1)}\}$. Note that under the stationary environment, for any $0\leq j\leq w-1$, $\tau_{k,1}^{(j)}$ is a random variable with the geometric distribution 
\[
\mathbb{P}(\tau_{k,1}^{(j)} =nw+j) = p(1-p)^{n-1},
\]
where $p=\mathbb{P}(S_{k,w} >b) $. Therefore, considering union bound we have that for any $k\in \mathcal{K}$
\[
\mathbb{P}(\tau_{k,1}\leq T) \leq w\left(1-(1-p)^{\floor{T/w}} \right).
\]
The remaining task is to find an upper bound for $p$.
Note that for any $k\in \mathcal{K}$, $S_{k,w}$ is a random variable with zero mean. We have by the McDiarmid's inequality and the union bound that 
\[
p \leq 2 \cdot \exp\left(-\frac{2b^2}{w}\right). 
\]
Combining the above analysis we conclude the result. 
\end{proof}

\subsection{Proof of Lemma~\ref{lem:detection}}
\label{sec:proof_detection}

\begin{proof}[Proof of Lemma \ref{lem:detection}]

Assume that $\delta_{\tilde{k}}^{(1)} \geq 2b/w+c$ for some $\tilde{k} \in \mathcal{K}$. Since the uniformly sampling scheme (step 2-4 of Algorithm \ref{alg:bandit}) guarantees that in any time interval with length larger than $L/2$ each arm is sampled at least $w/2$ times, conditioning on $\{ \tau_1 > \nu_1 \}$, we have that 
\begin{equation}
\begin{split}
&\mathbb{P}(\nu_1<\tau_1\leq \nu_1+L/2 \mid \tau_1>\nu_1)  \\
\geq& \mathbb{P}
\left(S_{\tilde{k},w}>b\right) \\
\geq& 1-2\exp\left(-\frac{(w|\delta_{\tilde{k}}^{(1)}|/2-b)^2}{w}\right) \\
\geq & 1-2\exp\left(-\frac{wc^2}{4}\right),
\end{split}
\end{equation}
where $S_{k,t}$ is defined in (\ref{eq:stat}) and we use McDiarmid's inequality in the second inequality. 
\end{proof}

\subsection{Proof of Lemma \ref{lem:DD}}
\label{sec:proof_dd}

\begin{proof}[Proof of Lemma \ref{lem:DD}]
First, define that $N = \ceil{b/\delta_{\tilde{k}}^{(1)}} \cdot \ceil{K/\gamma}$, we obtain a simple upper bound for the EDD as follows. 
\begin{equation}
\begin{split}
&\mathbb{E}[\tau_1 - \nu_1 \mid \nu_1<\tau_1\leq \nu_1+L/2 ] \\
=& \sum_{i=1}^{L/2} \mathbb{P}(\tau_1 \geq \nu_1+i\mid \nu_1<\tau_1\leq \nu_1+L/2) \\
\leq& N + \sum_{i=N}^{L/2} \left(\mathbb{P}(\tau_1 \geq \nu_1+i\mid \nu_1<\tau_1\leq \nu_1+L/2) \right).
\end{split}\nonumber
\end{equation}
Since the uniformly sampling scheme guarantees that we have at least $i/\ceil{K/\gamma}$ samples from each arm within $i$ time steps, we use McDiarmid's inequality and Lemma \ref{lem:detection} to have that 
\begin{equation}
\begin{split}
&\sum_{i=N}^{L/2} \left(\mathbb{P}(\tau_1 \geq \nu_1+i\mid \nu_1<\tau_1\leq \nu_1+L/2) \right) \\
=& \sum_{i=N}^{L/2} \frac{\mathbb{P}(\nu_1+i \leq \tau_1 \leq \nu_1+L/2 \mid \tau_1>\nu_1)}{\mathbb{P}(\nu_1 \leq \tau_1 \leq \nu_1+L/2 \mid \tau_1>\nu_1)} \\
\leq & \frac{1}{1-2\exp\left(-wc^2/4\right)} \cdot \sum_{i=N}^{L/2} 2\exp\left(-\frac{(i/\ceil{(K/\gamma)} \delta_{\tilde{k}}^{(1)}-b)^2}{w} \right) \\
\leq & \frac{\ceil{K/\gamma}}{1-2\exp\left(-wc^2/4\right)} \cdot \sum_{j=\ceil{b/\delta_{\tilde{k}}^{(1)}}}^{w/2} 2\exp\left(-\frac{(j \delta_{\tilde{k}}^{(1)}-b)^2}{w} \right).
\end{split}\nonumber
\end{equation} 
Define $q=\ceil{(w/2)\cdot \delta_{\tilde{k}}^{(1)}} -b$ and we have $q>1$ from the assumption that $\delta_{\tilde{k}}^{(1)}>2b/w+c$. Combining the above analysis, we have that
\begin{equation}
\begin{split}
& \left(1-2\exp\left(-wc^2/4\right)\right) \cdot \mathbb{E}[\tau_1 - \nu_1 \mid \nu_1<\tau_1\leq \nu_1+L/2 ] \\
\leq& N + \ceil{K/\gamma} \cdot \sum_{j=\ceil{b/\delta_{\tilde{k}}^{(1)}}}^{w/2} 2\exp\left(-\frac{(j\delta_{\tilde{k}}^{(1)}-b)^2}{w} \right) \\
\leq & N + 2\ceil{K/\gamma} \cdot \left(1+ \int_{1}^{q} \exp\left(-\frac{l^2}{w}\right)dl  \right) \\
\leq & N + 2\ceil{K/\gamma} \cdot \\
 & \left[1+\sqrt{w} \left( 1-\frac{1}{\sqrt{w}} + \int_1^{q/\sqrt{w}} \exp(-u^2)du \right)  \right] \\
\leq& N + 2\ceil{K/\gamma} \cdot\left(\sqrt{w} + \sqrt{w}\int_1^{q/\sqrt{w}} u\exp(-u^2)du \right) \\
\leq& \left(\ceil{b/\delta_{\tilde{k}}^{(1)}} + 3\sqrt{w}\right) \cdot \ceil{K/\gamma},
\end{split}\nonumber
\end{equation}
where we transform $l$ into $u=l/\sqrt{w}$ in the third inequality and we use the fact that $\exp(-u^2) \leq u\exp(-u^2), u\geq 1$ in the fourth inequality. 
On the other hand, by the definition of the conditioning event we also have that
\[
\left(1-2\exp\left(-wc^2/4\right)\right) \cdot \mathbb{E}[\tau_1 - \nu_1 \mid \nu_1<\tau_1\leq \nu_1+L/2 ] \leq L/2.
\]
Combining the above analysis we conclude the result. 
\end{proof}

\subsection{Proof of Theorem \ref{thm:bandit}}
\label{sec:proof_main}

\begin{proof}[Proof of Theorem \ref{thm:bandit}]
%
Recall that $L=w\ceil{K/\gamma}$. Algorithm \ref{alg:bandit} guarantees that in any time interval with length larger than $L$ each arm is sampled at least $w$ times. 
Define events $F_i = \{\tau_i>\nu_i\}, 1\leq i\leq M-1$. Define events $D_i = \{\tau_i\leq \nu_i+L/2\}, 1\leq i\leq M-2$ and event $D_{M-1} = \{\tau_{M-1} \leq T \}$. Therefore, the event $F_iD_i$ is the good event where the $i$th change can be detected correctly and efficiently. Define that $R(s) = \sum_{t=1}^s \max_{k\in \mathcal{K}} X_{k,t} - X_{A_t, t}$, then we have that $\mathcal{R}(T) = \mathbb{E}[R(T)]$. Equipped with the sequence of good events, we have that 
\begin{equation}
\begin{split}
\mathcal{R}(T) = \mathbb{E}[R(T)] 
\leq & \mathbb{E}[R(T) \mathbb{I}\{F_1\}] + T\cdot (1-\mathbb{P}(F_1)) \\
\leq & \mathbb{E}[R(\nu_1) \mathbb{I}\{F_1\}] + \mathbb{E}[R(T) - R(\nu_1)] + 1 \\
\leq & \tilde{C}_1 + \gamma \nu_1 + \mathbb{E}[R(T) - R(\nu_1)] +1.
\end{split}\nonumber
\end{equation}
Above, the second inequality is due to Lemma \ref{lem:falsealarm} that $\mathbb{P}(F_1) \geq 1-1/T$ provided that $b=[w\log(2KT^2)/2]^{1/2}$ and the third inequality is due to the bound in the end of Lemma \ref{lem:stationary_bandit}, which is the bound for the UCB algorithm in a stochastic bandit setting. 

The next step is to bound $\mathbb{E}[R(T) - R(\nu_1)]$. Using the law of total expectation, we have that
\begin{equation}
\begin{split}
& \mathbb{E}[R(T) - R(\nu_1)] \\
\leq &  \mathbb{E}[R(T) - R(\nu_1) \mid F_1D_1] 
+ T \cdot (1-\mathbb{P} (F_1D_1)) \\
\leq &  \mathbb{E}[R(T) - R(\nu_1) \mid F_1D_1] + 2,
\end{split}\nonumber
\end{equation} 
where the last inequality is due to Lemma \ref{lem:detection} that we have $\mathbb{P}(D_1 \mid F_1) \geq 1-1/T$ provided that $c=2\sqrt{\log(2T)/w}$ and the fact that $\mathbb{P}(F_1D_1) = \mathbb{P}(D_1 \mid F_1) \cdot \mathbb{P}(F_1)$ for any probability measure $\mathbb{P}$. 

Therefore, the remaining task is to bound $ \mathbb{E}[R(T) - R(\nu_1) \mid F_1D_1]$. Denote $\tilde{\mathbb{E}}$ as the expectation according to the piecewise-stationary bandit starts from the second segment. Further splitting the regret, we have that 
\begin{equation}
\begin{split}
&\mathbb{E}[R(T) - R(\nu_1) \mid F_1D_1]  \\
\leq & \mathbb{E}[R(T) - R(\tau_1) \mid F_1D_1] 
+ \mathbb{E}[R(\tau_1) - R(\nu_1) \mid F_1D_1] \\
\leq &  \tilde{\mathbb{E}}[R(T-\nu_1) ] 
+  \mathbb{E}[\tau_1 - \nu_1 \mid F_1D_1] \\
\leq & \tilde{\mathbb{E}}[R(T-\nu_1) ] 
+ \min(L/2, (\ceil{b/\delta^{(1)}} + 3\sqrt{w}) \cdot \ceil{K/\gamma})/\left(1-1/T\right)
\end{split}\nonumber
\end{equation}
where the second inequality is due to the renewal property given that the whole algorithm restarts in the time interval between $\nu_1$ and $\nu_1+L/2$ and the last inequality is due to Lemma \ref{lem:DD} by setting $c=2\sqrt{\log(2T)/w}$.

Combining the above analysis, we bound the regret in a recursive manner as follows (assuming $T\geq 2$):
\begin{equation}
\begin{split}
&\mathbb{E}[R(T)] \leq \tilde{\mathbb{E}}[R(T-\nu_1)] + \tilde{C}_1 \\
& + \gamma \nu_1+  2\min(L/2, (\ceil{b/\delta^{(1)}} + 3\sqrt{w}) \cdot \ceil{K/\gamma}) + 3.
\end{split} \nonumber
\end{equation}
The recursive manner means that we can apply the same method to bound $ \tilde{\mathbb{E}}[R(T-\nu_1)]$, by conditioning on the event $D_2F_2$. Repeating this procedure $M-1$ times, we obtain that
\begin{equation}
\begin{split}
\mathbb{E}[R(T)] \leq & \textstyle\sum_{i=1}^M \tilde{C}_i + \gamma T \\
+ &  \textstyle\sum_{i=1}^{M-1}\frac{2K \cdot \min(\frac{w}{2},\ceil{\frac{b}{\delta^{(i)}}}+3\sqrt{w})}{\gamma}+ 3M.
\end{split}\nonumber
\end{equation}

\end{proof}

\end{document}